\documentclass[twoside,11pt]{article} 

\usepackage{geometry}                
\usepackage[parfill]{parskip}    
\usepackage{graphicx,subfigure}
\usepackage{amssymb, amsmath, amsthm}
\usepackage{epstopdf}
\usepackage{multibib}
\usepackage{mathrsfs}
\usepackage{color}
\usepackage{algorithm}
\usepackage{algorithmic}
\usepackage{bbm}

\usepackage{tikz}
\usetikzlibrary{shapes}
\usetikzlibrary{intersections}
 
 \usepackage[T1]{fontenc}
 \usepackage[utf8]{inputenc}
 \usepackage{authblk}

\usepackage{color}

\newcommand{\R}{\mathbb{R}}

\newcommand{\E}{\mathbb{E}}

\newcommand{\OO}[1]{\operatorname{O}\left(#1\right)}

\newtheorem{theorem}{Theorem}

\newtheorem{conjecture*}{Conjecture}

\title{Function Driven Diffusion for Personalized\\ Counterfactual Inference}

\author{Alexander Cloninger$^{1,2}$\thanks{Email: alexander.cloninger@yale.edu}}
\affil{Applied Mathematics Program$^1$ and Center for Outcomes Research \& Evaluation$^2$}
\affil{Yale University, New Haven, CT}

\renewcommand\footnotemark{}

\date{}                                           

\begin{document}
	\maketitle

\begin{abstract}
We consider the problem of constructing diffusion operators high dimensional data $X$ to address counterfactual functions $F$, such as individualized treatment effectiveness.  We propose and construct a new diffusion metric $K_F$ that captures both the local geometry of $X$ and the directions of variance of $F$.  The resulting diffusion metric is then used to define a localized filtration of $F$ and answer counterfactual questions pointwise, particularly in situations such as drug trials where an individual patient's outcomes cannot be studied long term both taking and not taking a medication.  We validate the model on synthetic and real world clinical trials, and create individualized notions of benefit from treatment.
\end{abstract}

 \section{Introduction}
 
 We address the problem of building a metric on a high dimensional data set $X=\{x_i\}\subset \R^d$ that is smooth with respect to an external nonlinear function $F$.  They types of functions we consider arise from \emph{counterfactual questions}, such as ``would this patient benefit or be hurt from medication, given their history and baseline health?''.  In medical studies, these functions adopt the interesting feature that they cannot be evaluated pointwise, since an individual patient's outcomes cannot be studied long term both taking and not taking a medication.
 
 
 \subsection{Mathematical Approach}

 In the case of an $A/B$ treatment study, we denote the risk on treatment $A$ as $Y_A(x)$ and treatment $B$ as $Y_B(x)$.  The quantity $Y_A(x) - Y_B(x)$ is known as the individual treatment effectiveness.  The treatment group is denoted $T_X = 1$ if $x$ is in treatment $B$, $0$ otherwise.  A current method of dealing with treatment effectiveness is the \emph{Cox proportional hazard model} from \cite{coxModel}.  In it, we let $\lambda_0(t)$ be the common baseline hazard function which describes the risk of an outcome at each time step independent of treatment.  Within treatment groups, the hazard function for the Cox proportional hazard model takes the form
\begin{eqnarray*}
	\lambda(t) = \lambda_0(t) e^{T_X \alpha + \vec{\beta} X},
\end{eqnarray*}
where $T_X\in\{A,B\}$ is an indicator function for which treatment patient $x$ was in.  This makes the associated survival distribution
\begin{eqnarray*}
	P(W\ge w) = exp\left[-e^{X_\alpha Y_1(X) + (1-X_\alpha) Y_0(X)}\int_0^w \lambda_0(t) dt \right].
\end{eqnarray*}

Also assume there is a random censoring model, which means that people leave the trial at random times throughout the process.  This means we don't observe the true outcome time $W$ of each patient, but instead we observe the leave time $t = \min(W,C)$, where $C$ and $W$ are independent and $C$ represents time to censorship.  The indicator function $D$ of whether $W\le C$ is known, as well.

A patient personalized version of this model would be
\begin{eqnarray*}
	\lambda(t|x) = \lambda_0(t|x) e^{(1-T_X) Y_A(X) + T_X Y_B(X)},
\end{eqnarray*}
which now allows the benefit or detriment of the drug to be patient specific.

In this model of personalized risk, $Y_A$ and $Y_B$ are unknowable pointwise in a drug trial since each $x$ only takes one of the drugs in $\{A,B\}$.  So we estimate $Y_B(X) - Y_A(X)$ in a neighborhood by assuming that locally, patients satisfy a proportional hazard model, with
\begin{eqnarray*}
	\lambda_0(t|z) \approx \lambda_0(t|x), \textnormal{ } Y_A(z) \approx Y_A(x),\textnormal{ } Y_B(z) \approx Y_B(x),  
	\textnormal{ for } z\in\mathscr{N}(x)= \{z\in\R^m : \rho(x,z)<\epsilon\},
\end{eqnarray*}
for some metric $\rho$, which we discuss further in Section \ref{weightedTree}.  Thus we can assume that everyone in the neighborhood shares a common baseline risk, and the relative benefit of treatment is a constant multiple of that risk.  This allows us to run a cox proportional hazard model on $z\in\mathscr{N}(x)$ by fitting $\alpha$ to
\begin{eqnarray*}
	\lambda(t|z) = \lambda_0(t) e^{\alpha T_Z}, & z\in\mathscr{N}(x),\\
	F(x) = \alpha. & 
\end{eqnarray*}
Estimate of $\alpha$ for each neighborhood can be done in several ways.  If we only observe $D$ (i.e. whether or not the patient had an outcome before leaving the trial), then $\alpha$ is estimated through method of moments between the two treatment groups.  If we observe the actual outcome time $t$ along with $D$, $\alpha$ is estimated through partial likelihood maximization.  As a note, while partial likelihood maximization uses more information and thus should result in a better estimate, convergence guarantees are more difficult to derive.  We present certain guarantees for both approaches in Section \ref{locHR}.

This means $F(x)$ reflects the amount a patient is positively or negatively affected by a drug, and can be used to approximate $Y_B(X) - Y_A(X)$. The problem turns into a metric discovery problem of 
determining a metric $\rho$ that learns the level sets of $F$. This is akin to finding pockets of people, based only on baseline information $X$, that are at much higher risk (or lower risk) on drug $A$ than they are on drug $B$.  Discovery of the metric $\rho(x,y)$ then allows for an analysis of ``types'' of responders and non-responders.

 We view these types of functions $F:X\rightarrow \R$ as functions that can only be evaluated on large subsets of the data.  In other words, $F(E)$ for $E\subset X$ is only computable when $|E|\ge c>0$.  We build an algorithm which constructs a metric $\rho:X\times X\rightarrow \R^+$ such that
\begin{eqnarray*}
	| F(E) - F(E') | < C \rho(E,E'),
\end{eqnarray*}
for a small constant $C$, where $\rho(E,E') = \max \{\rho(x,y) : x\in E, y \in E'\}$.  
In other words, the metric $\rho$ does not only consider the geometry of the space $X$, but also the geometry and the properties of the function $F$ being studied.

 The purpose of computing $\rho$ is two-fold:
 \begin{enumerate}
 \item this discovers the intrinsic organization of $X$ which dictates changes in $F$.  This makes any subsequent clustering or analysis done using $\rho$ reflect the level sets of $F$, as well as the intrinsic structure of $X$.  The reason for doing this is that $F$ may not be smooth with respect to the intrinsic geometry of the space, but has structure that is well described by a subset of the features.  Also,
 \item this allows for simple estimation of $F$ at a finer scale than it is reliable naively.  Using $\rho$, we are able to construct an estimate of $F$, which we call $\widehat{f}$, which can be evaluated pointwise via
 \begin{eqnarray*}\label{eq:limitApproxF}
\widehat{f}(x) = \lim\limits_{\epsilon\rightarrow 0} F(\mathscr{N}_d^\epsilon(x)), & \mathscr{N}_\rho^\epsilon(x) = \{z\in\R^m : \rho(x,z)<\epsilon\}.
\end{eqnarray*}
One can also define a multi-scale decomposition of $F$ via 
\begin{eqnarray*}\label{eq:waveletApproxF}
\widehat{f}(x) = \sum_{\epsilon_i} f_{\epsilon_i}(x), & f_\epsilon(x) = F(\mathscr{N}_d^{\epsilon/2}(x)) - F(\mathscr{N}_d^{\epsilon}(x)).
\end{eqnarray*}
The key in both these approximations is that, provided $F$ is smooth with respect to $\rho$, the approximating $\epsilon-$neighborhood will have a large radius about level sets of $F$.  This increases the number of points $\{x_i\}$ in $\mathscr{N}_\rho^\epsilon(x)$ for a fixed $\epsilon$, making the approximations more accurate than those generated by taking an isotropic ball of radius $\epsilon$ about $x$.
 \end{enumerate}

\subsection{Main Contributions}

The study of individualized treatment effects has recently been considered with linear lasso models of \cite{qian2011performance}, and linear logistic models with AdaBoost of \cite{kang2014combining}.  A number of models have been built to predict outcomes from a single treatment, but high risk for an outcome does not necessarily imply treatment benefit, as seen in  \cite{janes2011measuring,janes2014approach}.  While these models provide useful treatment recommendations, they project to a one dimensional function space and interpretability is limited to the non-zero coefficients of the model.  

While we are interested in determining a treatment recommendation, we are also interested in the question of characterizing the level sets of a treatment effect.  Diffusion embeddings provide a non-linear framework to map out the data into a continuum of varying treatment effectiveness.  Using a diffusion metric, one can determine variability of types of patients that similarly benefit from treatment (or lack of treatment).

Function regularized diffusion has been considered when $F$ can be evaluated pointwise by \cite{szlam2008}.  
We have also previously examined building non-linear features of functions $F$ that cannot be evaluated pointwise and subsequently organize these features in \cite{cloninger2015}, as well as regression of non-linear Cox proportional hazard functions in  \cite{jared2016}.


The main contributions of this work are:
\begin{itemize}
\item the ability to build a function regularized diffusion metric without the ability to evaluate $F$ pointwise,
\item interpolation of a function regularized diffusion metric to new points where $F$ is unknown, and
\item the use of function regularized diffusion to define pockets of responders and non-responders to a given treatment.
\end{itemize}

 This paper is organized as follows.  Section \ref{background} gives background descriptions of the tools we reference throughout the paper, including diffusion maps, and hierarchical cluster treesSection \ref{weightedTree} details the function weighted trees used to generate $\rho$, as well as the introduces the notion of estimating a data point's personalized  function estimate.  Section \ref{locHR} discusses the guarantees that can be given for personalized treatment effect, as well as convergence rates.  Section \ref{applications} applies and validates our algorithm on several datasets of synthetic patients, and discovers the original ground truth metric.  We also examine the algorithm on real world patient data and examine validation schemes.

 \section{Background}\label{background}
 
 In this section, we discuss previous research that considers organization of points.  This considers $M\in \R^{n\times m}$ as a data matrix of $n$ points and $m$ features per point.  Denote the rows of $M$ by $X$ (the set of points), and the set of columns by $Y$ (the set of features or questions).  For this section, there is no external function $f$ being considered.
  
 \subsection{Diffusion Geometry}
 Diffusion maps is a manifold learning technique based on solving the heat equation on a data graph, as in \cite{coifman2006}.  It has been used successfully in a number of signal processing, machine learning, and data organization applications.  We will briefly review the diffusion maps construction.
 
 Let $X = \{x_1, ..., x_n\}$ be a high dimensional dataset with $x_i \in \R^m$.  A data graph is constructed with each point $x_i$ as a node and edges between two nodes with weights $k(x_i, x_j)$.  The affinity matrix $K_{i,j} = k(x_i, x_j)$ is required to be symmetric and non-negative.  Common choices of kernel are the gaussian
 \begin{eqnarray*}
 k(x_i, x_j)  = e^{-\frac{\|x_i - x_j\|^2_2}{2\sigma^2}},
 \end{eqnarray*}
 or positive correlation
 \begin{eqnarray*}
 k(x_i, x_j) = \max\left( \frac{ \langle x_i, x_j \rangle}{\|x_i\|\|x_j\|} , 0\right).
 \end{eqnarray*}
$K$ can be computed using only nearest neighbors of $x_i$ such that $k(x_i, x_j) \ge \tau > 0$.
 
 Let $D_{i,i} = \sum_j k(x_i,x_j)$.  We normalize kernel $K$ to create a Markov chain probability transition matrix 
 \begin{eqnarray*}
 P = D^{-1} K.
 \end{eqnarray*}
 The eigendecomposition of $P$ yields a sequence of eigenpairs $\{(\lambda_i, \phi_i)\}_{i=0}^{n-1}$ such that $1=\lambda_0 \ge \lambda_1 \ge 	...$
 
 The diffusion distance $d^t_{DM}(x_i , x_j)$ measures the distance between two points as the probability of points transitioning to a common neighborhood in some time $t$.  This gives
 \begin{eqnarray*}
 d^t_{DM}(x_i, x_j) = \sum\limits_{x_k\in X} \left(P^t(x_i, x_k) - P^t(x_j, x_k)\right)^2 = \sum_{k\ge1} \lambda_k^{2t} \left(\phi(x_i) - \phi(x_j)\right)^2.
 \end{eqnarray*}
 
 Retaining only the first $d$ eigenvectors creates an embedding $\Phi_t:X\rightarrow \R^d$ such that
 \begin{eqnarray*}
 \Phi_t : x_i \rightarrow [\lambda_1^t \phi_1(x_i) , ..., \lambda_d^t \phi_d(x_i)].
 \end{eqnarray*}

 Figure \ref{fig:twoCircles} shows a two dimensional example dataset and the data graph generated on the points.  We also see the low frequency eigenfunctions on the data graph, and the diffusion embedding $\Phi_t$.
 
 \begin{figure}[!h]
 \footnotesize
 \begin{tabular}{cc}
 \includegraphics[width=.4\textwidth]{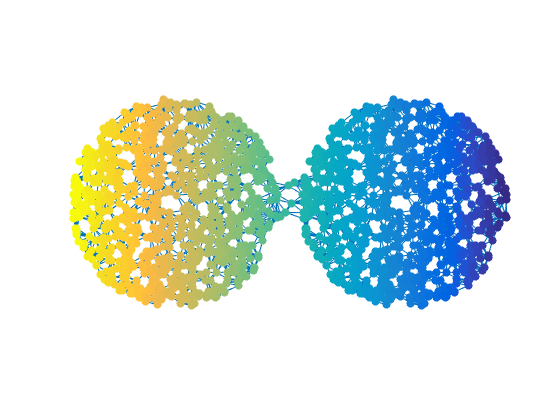} & 
 \includegraphics[width=.4\textwidth]{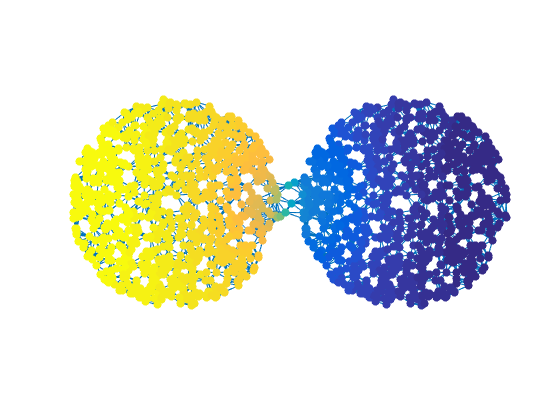} \\
 (a) Data Graph colored by x-coordinate & (b) Data graph colored by $\phi_1$\\
 \includegraphics[width=.4\textwidth]{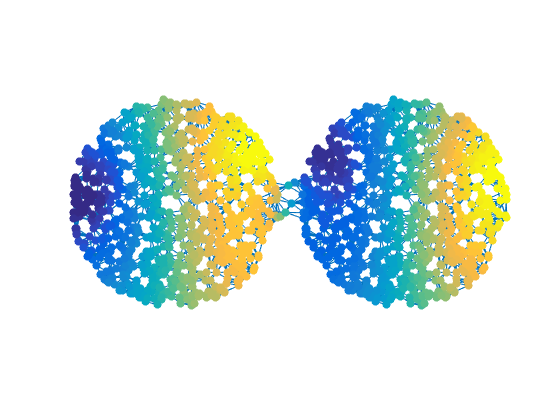} & 
 \includegraphics[width=.4\textwidth]{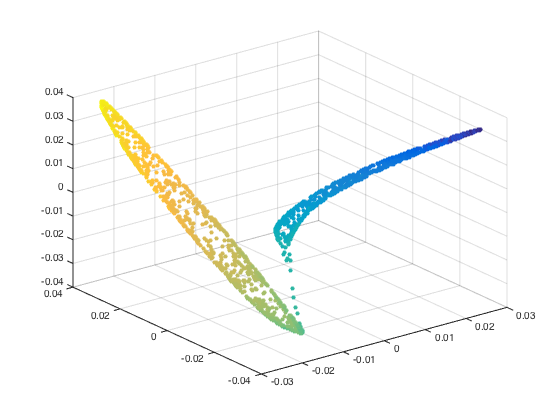} \\
 (c) Data graph colored by $\phi_2$ & (d) $\Phi_t$ colored by x-coordinate
 \end{tabular}
 \caption{Toy example to demonstrate relationship between geometry of dataset and eigenfunctions of graph.  Dataset is only 2D for ease of visualization, algorithm is equally valid on high-dimensional dataset.}\label{fig:twoCircles}
 \end{figure}
 
 {\bf Remark:} The diffusion time $t$ is a continuous variable, which can be thought of as the degree to which $\Phi_t$ is a low-pass filter.  For small $t$, more of the high-freqncy eigenfunctions are given non-trivial weight.  For large $t$, the embedding is mostly concentrated on the low-frequency eigenfunctions that vary slowly across the data.

 \subsection{Hierarchical Tree From Diffusion Distance}
  The main idea behind bigeometric organization is to construct a coupled geometry via a partition tree on both the data points and the features.   A partition tree is effectively a set of increasingly refined partitions, in which finer child partitions (i.e. lower levels of the tree) are splits of the parent folder which attempt to minimize the inter-folder variability.
 
 Let $X\subset \R^m$ be a dataset of points, and $\Phi_t : X \rightarrow \R^d$ be a diffusion embedding with corresponding diffusion distance $d_{DM}^t : X\times X \rightarrow \R^+$.  A partition tree on $X$ is a sequence of $L$ tree levels $\mathscr{X}^\ell$, $1\le \ell\le L$.  Each level $\ell$ consists of $n(\ell)$ disjoint sets $\mathscr{X}^\ell_i$ such that 
 \begin{eqnarray*}
 X = \bigcup_{i=1}^{n(\ell)}  \mathscr{X}^\ell_i.
 \end{eqnarray*}
 Also, we define subfolders (or children) of a set $\mathscr{X}^\ell_i$ to be the indices $I^{\ell+1}_i \subset\{1, ..., n(\ell+1)\}$ such that
 \begin{eqnarray*}
 \mathscr{X}^\ell_i = \bigcup_{k \in I^{\ell+1}_i} \mathscr{X}^{\ell+1}_k.
 \end{eqnarray*}
For notation, $\mathscr{X}^1 = X$ and $\mathscr{X}^L_i = \{x_i\}$.  See Figure \ref{fig:treePlot} for a visual breakdown of $X$.

 This tree can be in two ways:
 \begin{enumerate}
 \item \emph{Top-down:} Taking the embedded points $\Phi_t(X)\subset \R^d$, the initial split $\mathscr{X}^2$ divides the data into 2 (or k) clusters via k-means or some clustering algorithm.  Each subsequent folder is then split into 2 (or k) clusters in a similar way, until each folder contains a singleton point.
 
 \item \emph{Bottom-up:} Taking the embedded points $\Phi_t(X)\subset \R^d$, the bottom folders $\mathscr{X}^{L-1}$ are determined by choosing a fixed radius $\epsilon$ and covering $\Phi_t(X)$ with balls of radius $\epsilon$.  Each subsequent level of the tree is then generated as combinations of the children nodes that are ``closest'' together under the distance $d^t_{DM}$.

 \end{enumerate}
 
\begin{figure}[!h]
\begin{center}
\includegraphics[width=.6\textwidth]{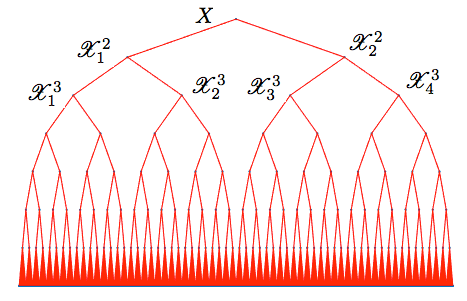}
\end{center}
\caption{Breakdown of $X$ into folders.}\label{fig:treePlot}
\end{figure}

{\bf Remark:} It is important to note that whether one chooses a top-down or bottom-up approach, the fact that the clustering occurs on the diffusion embedding $\Phi_t(X)$ makes the resulting tree, by definition, a ``bottom-up'' geometry.  This is because the embedding and diffusion distance is built off of local similarities alone, meaning that the resultant geometry and partition tree are based on properties of the underlying dataset and manifold rather than the ambient dimension and a naive Euclidean distance in $\R^m$.

\section{Weighted and Directional Trees Without Pointwise Function Evaluation}\label{weightedTree}
Let us denote our data space $X\subset \R^m$.  In its most general form, we have an external function $F:X \rightarrow \R$ which cannot be evaluated pointwise.  $F$ can only be evaluated on large subsets $E\subset X$.  We define the pointwise estimate of $F$ to be $\widehat{f}(x)$, as in \eqref{eq:limitApproxF}.  This is done by defining a type of locally weighted distance to incorporate estimates of a feature's power to discriminate $F$ in different half spaces.  The details on this method are in Section \ref{weightedEMD}.  An overview of the approach is in Algorithm \ref{mainAlgorithm}.

\begin{algorithm}[h!]
\caption{Calculate Function Weighted Metric}\label{mainAlgorithm}
\begin{tabular}{ll}
{\bf Required:} & Training points $\{x_i\}_{i=1}^n \in\R^{m}$ \\
 & Function $F$ to be evaluated on $E\subset \R^m$ \\
 
{\bf Result:} & $\Phi_t: \{x_i\}_{i=1}^n \rightarrow \R^d$ such that range$(\widehat{f}(x))$ is large, where \\
\end{tabular}
\begin{eqnarray*}
\widehat{f}&:&\{x_i\}_{i=1}^n \rightarrow \R \\
x &\mapsto& F(\mathscr{N}_{d_{F}^t}^\epsilon(x)) 
\end{eqnarray*}

\begin{enumerate}

\item Build a diffusion embedding of the points $\Phi_t(X)$ and a hierarchical tree $\mathscr{X}$


\item Build a tree $\mathscr{Y}$ that determines the local coordinate feature weights (see Section \ref{weightedEMD})

\item Build a new diffusion embedding of the points $\Phi_t(X)$ and a hierarchical tree $\mathscr{X}$ based on the kernel in \eqref{eq:weightedEMDKernel}

\item Iterate between the points and the features until embedding $\Phi_t(X)$ and tree $\mathscr{Y}$ are stable

\item Define pointwise neighborhood $\mathscr{N}_{d_{F}^t}^\epsilon(x) = \{z\in\R^m : \|\Phi_t(x) - \Phi_t(z)\|_2<\epsilon\}$ and function estimate $\widehat{f}(x)$
\end{enumerate}
\end{algorithm}


\subsection{Weighted Trees}\label{weightedEMD}

 Let $M\in \R^{n\times m}$ be the data matrix and $F$ be the integral operator of interest.  Denote the rows of $M$ by $X$ (the set of points), and the set of columns by $Y$ (the set of features or questions).  We wish to build feature weights on each folder of $\mathscr{X}$ that maximally separate $F$.  The algorithm is as follows:
  \begin{enumerate}
 
 \item Assume the tree $\mathscr{X}$ is known and separates $X$ into hierarchical nodes.  Fix a node $\mathscr{X}^\ell_i$.
 
 \item  For each element $y\in Y$, we split $y$ into $k$ intervals $[a_j, a_{j+1})$ and bin the elements of $\mathscr{X}^\ell_i = \{h_j\}_{j=1}^{k}$ such that
 \begin{eqnarray*}
 x\in h_j &\iff& x(y) \in [a_j, a_{j+1}) \textnormal{ and } x\in \mathscr{X}^\ell_i.
 \end{eqnarray*} 
 Question $y$ is then assigned a local weight for its ability to discriminate $F$ by
 \begin{eqnarray}\label{eq:1Dweight}
 w^\ell_i  (y) = \sum_{j=1}^k \frac{ |h_j|}{|\mathscr{X}^\ell_i |} \cdot |F(h_j) - \bar{F}|^2,
  \end{eqnarray}
where $\bar{F}$ is the weighted mean across all bins.  

 \item Now that every node of the tree $\mathscr{X}$ has local feature weights, we calculate the local weights at a point $x_i$ by 
 \begin{eqnarray*}
w_{x_i}(y) = \sum_\ell 2^{-\alpha\ell} w_{x_i}^\ell(y), &\textnormal{where}& w_{x_i}^\ell(y) = w_i^\ell(y) \textnormal{ for } x_i \in \mathscr{X}_i^\ell.
 \end{eqnarray*}
 These weights create a diagonal matrix $W_{x_i}$ where $W_{x_i}[y,y] = (w_{x_i}(y) + \lambda)^{-1}$ for a small positive constant $\lambda$.
  
 \item The kernel function $k:X\times X\rightarrow \R^+$ is then 
 \begin{eqnarray}\label{eq:weightedEMDKernel}
 k(x_i, x_j) =  \frac{e^{- (x_i - x_j)^\intercal ( W_{x_i} + W_{x_j} )^{-1} (x_i - x_j) / \sigma^2}}{\sqrt{\det(W_{x_i}+ W_{x_j})}}.
 \end{eqnarray}
The normalization in the denominator is needed to guarantee $k$ is positive semi-definite.

\end{enumerate}

\begin{theorem}\label{thm:posdefKernel}
The kernel $k:X\times X \rightarrow \R^+$ from \eqref{eq:weightedEMDKernel} is positive semi-definite.
\end{theorem}

The proof of Theorem \ref{thm:posdefKernel} is in Appendix \ref{appendix:posdefKernel}.

Because $k$ is positive semi-definite, we can compute the embedding of the data $\Phi_t(X)$, and induce a new diffusion metric on the data,
\begin{eqnarray*}
d_{F}^t(x,y) = \|\Phi_t(x) - \Phi_t(y)\|_2.
\end{eqnarray*}
This, in turn, allows us to define $\widehat{f}: \{x_i\} \rightarrow \R$ as an estimate to $f(x_i)$, where
\begin{eqnarray*}
\widehat{f}(x) = F(\mathscr{N}_{d_{F}^t}^\epsilon(x)), &\textnormal{ where } \mathscr{N}_{d_{F}^t}^\epsilon(x) = \{z\in\{x_i\} : \|\Phi_t(x) - \Phi_t(z)\|_2<\epsilon\}.
\end{eqnarray*}

\subsection{Interpolation and Leave Out Validation}\label{leaveOutValidation}
The metric $d_{F}^t$ and function estimate $\widehat{f}$ can easily be extended to new points $z \not\in \{x_i\}$ not in the training data.  This is done by building an asymmetric affinity matrix to the training data, which can be thought of as a reference set.  The approach is an application of \cite{kushnir2012anisotropic}, which we briefly outline here.

Let $X$ be training data, and $F$ defined on subsets of $X$.  Let $Z$ be testing points on which $F$ is not defined a priori.  
Define $k: \left(X\cup Z\right) \times X \rightarrow \R^+$ to be
\begin{eqnarray*}
k(z,x) =  \frac{e^{- (z - x)^\intercal W_{x}^{-1} (z - x) / \sigma^2}}{\sqrt{\det(W_{x})}}, & z\in Z\cup X, \textnormal{ } x\in X.
\end{eqnarray*}
With the normalization matrices $(D_1)_{ii} = \sum_j k(z_i, x_j)$ and $(D_2)_{ii} = \sum_j k(z_j, x_i)$, we set
\begin{eqnarray*}
A = D_1^{-1/2} k D_2^{-1/2},
\end{eqnarray*}
and take the eigendecomposition of the small matrix $A^* A= \Psi \Sigma \Psi^*$.  This gives an embedding of the reference points $X$.  Then eigendecompsition of the entire set of points $X\cup Z$ is estimated by
$\Phi = A \Psi$.  The details of the extension algorithm can be found in Algorithm \ref{leaveOutAlgorithm}.

\begin{algorithm}[h!]
\caption{Nearest Neighbor Function Estimation}\label{leaveOutAlgorithm}
\begin{tabular}{ll}
{\bf Required:} & Training points $\{x_i\}_{i=1}^n \in\R^{m}$ \\
 & Function $F$ to be evaluated on $E\subset \{x_i\}$ \\
 & Testing points $\{z_i\}_{i=1}^N \in \R^m$\\
 
{\bf Result:} & Pointwise function estimate on testing points $\widehat{f}(z_i)$\\
\end{tabular}
\begin{eqnarray*}
\end{eqnarray*}

\begin{enumerate}

\item Build stable function weighted diffusion embedding of the features $\Phi_t(Y)$ and a hierarchical tree $\mathscr{Y}$ via Algorithm \ref{mainAlgorithm} using the training points $\{x_i\}_{i=1}^n $

\item Build a function weighted diffusion embedding $\widehat{\Phi_t}(X\cup Z)$ using the weighted embedding via reference set algorithm (see Section \ref{leaveOutValidation})

\item For each testing point $z$, define the pointwise training neighborhood $\mathscr{N}_{d_{F}^t}^\epsilon(z) = \{x\in\{x_i\} : \|\Phi_t(x) - \Phi_t(z)\|_2<\epsilon\}$ and function estimate 
\begin{eqnarray*}
\widehat{f}(z) = F(\mathscr{N}_{d_{F}^t}^\epsilon(z))
\end{eqnarray*}

\end{enumerate}
\end{algorithm}

Algorithm \ref{leaveOutAlgorithm} can be thought of as generating an optimized metric for a k-nearest neighbor search.  There could be better mechanisms of classification and regression for predicting $\widehat{f}$, ranging from support vector machines in \cite{Boser1992}, to various types of linear regression, such as Elastic Net from  \cite{elasticNet2005}.  These choices are application and function specific, which is why we remain with a simple nearest neighbor interpolation.  The key is that the metric $d_{F}^t$ agrees with the intrinsic geometry of the data.

It is also important to run leave out validation of the algorithm to insure no overfitting of the data.  As the algorithm is semi-supervised and weights variables according to their discriminatory power, it is possible to give high weight to features which are spuriously correlated with the function.  This makes it crucial to run N-fold cross validation of the data to ensure that the predicted $\widehat{f}(z)$ are good estimates to the true function.

\section{Localized Hazard Ratio Estimation}\label{locHR}

The aim of the construction of the metric $d_{F}^t$ from Section \ref{weightedTree} is to construct a metric that differentiates level sets of the treatment risk.  This implies that, in a neighborhood $\mathscr{N}_\epsilon(x)$, we can assume that $|(Y_1(Z) - Y_0(Z)) - (Y_1(X) - Y_0(X))|<\delta$.  This implies we can estimate the local treatment effect $\widehat{\alpha}$ from the model
\begin{eqnarray}
	\lambda(t|Z) &=& \lambda_0(t) e^{T_Z Y_1(Z) + (1-T_Z) Y_0(Z)}\\
	 & \approx &  \lambda_0(t) e^{T_Z \alpha + Y_0(Z)},\label{eq:localApprox}
\end{eqnarray}
for $Z\in\mathscr{N}_\epsilon(X)$.  This approximation is because $\alpha = Y_1(Z) - Y_0(Z)$ does not vary more than $\delta$ in $\mathscr{N}_\epsilon(X)$.

Now assume we fit the false model to simply estimate treatment effectiveness, which is necessary given no knowledge of the model assumptions for $Y_0(Z)$ locally.  What we can assume is that, given $Z\in\mathscr{N}_\epsilon(X)$, $Y_0(Z)$ cannot vary too much within a small neighborhood.  

There are two regimes in which we study this question of estimating $\alpha$.  In either situation, we have a censoring model $C\sim q$ and observe an outcome $W$ only if $W\le C$.  When that's the case, we denote $D=1$, with $D=0$ otherwise.

If we only observe $D$ (i.e. whether or not the patient had an outcome before leaving the trial), then $\alpha$ is estimated through method of moments between the two treatment groups.  If we observe the actual outcome time $t$ along with $D$, $\alpha$ is estimated through partial likelihood maximization.  We provide results for both, with the stronger and more concrete results coming for observation only of the binary outcome variable $D$.

In both settings, we assume that the patient risk over time is dictated by \eqref{eq:localApprox}.  We also assume that there is some censoring model $C$ for each patient, which is a random variable independent of \eqref{eq:localApprox} that dictates when a patient decides to leave the trial, if they are still alive.

\subsection{Binary Observation of Outcome with Censoring}

In this scenario, we only observe $D$ for each patient.  This means we know whether they had an outcome before they left the trial, but not the time at which the outcome occurred.

To create an estimate $\widehat{\alpha}$, we use a method of moments approach.  That is, within the neighborhood $\mathscr{N}_\epsilon(X)$, we look at the empirical estimate
\begin{equation}\label{eq:methodMoments}
\footnotesize
\frac{1}{|\mathscr{N}_\epsilon(X)\cap \{T_X=1\}|}\left(\sum_{z\in\mathscr{N}_\epsilon(X) \cap \{T_X=1\}} D_X\right) - \frac{1}{|\mathscr{N}_\epsilon(X)\cap \{T_X=0\}|}\left(\sum_{z\in\mathscr{N}_\epsilon(X) \cap \{T_X=0\}} D_X\right).
\end{equation}

\normalsize
We borrow and modify results from \cite{gail1984biased} about small variation of misspecified models. 
For notation, let 
\begin{eqnarray*}
\Pi(X,T_X) = P(D = 1 | X,T_X), & & \Pi(T_X) = P(D = 1 | T_X) = \E_X(\Pi(X,T_X)).
\end{eqnarray*}

\begin{theorem}\label{thm:localHR}
Let the survival model satisfy \eqref{eq:localApprox}, and the $P(T_X = 1) = p$ for $0<p<1$.  Assume we use a method of moments estimation of the misspecified model
\begin{eqnarray*}
	\lambda(t|X) &=& \lambda_0(t) e^{T_X \alpha}.
\end{eqnarray*}
Let us further assume we only observe an indicator of outcome $D$.

Then the method of moments estimate converges at a rate of $ N^{-1/2}  C_{p,\alpha,Y_0,\lambda_0,q}$ to $\alpha^*$ for some finite constant $C$ that depends on $p$, $\alpha$, $\lambda_0$, the non-treated risk model, and the censoring model.  The estimate converges to $\alpha^*$, which satisfies
\begin{eqnarray*}
\alpha^*= \alpha + \log\left(\frac{\Pi(1) \E_X\left[\Pi(X,0) e^{ - Y_0(X)} \right]}{\Pi(0) \E_X\left[\Pi(X,1) e^{- Y_0(X)} \right]}\right).
\end{eqnarray*}

Moreover, if $Y_0(X)$ is well approximated locally by its first order Taylor expansion $Y_0(Z) = \mu + \beta X + \OO{\Sigma_X}$ for small $\beta$, then we can reduce the $\log$ term to further show
\begin{eqnarray*}
	|\alpha^* - \alpha| < \frac{1}{2} \beta'\Sigma_X \beta \cdot |R(\alpha) - R(-\alpha) |
\end{eqnarray*}
where
\begin{eqnarray*}
	R(x) = 2 \phi'(x) / \phi(x), &\textnormal{for}& \phi(x) = \E(D | x),
\end{eqnarray*}
is a constant that depends only on the size of $\alpha$ and the censoring model.  Note also that this implies $\alpha>0 \implies \alpha^*>0$, $\alpha<0 \implies \alpha^*<0$, and $\alpha=0 \implies \alpha^*=0$.

\end{theorem}

The proof of Theorem \ref{thm:localHR} is in Appendix \ref{appendix:localHRProof}.

\subsection{Continuous Time to Outcome with Censoring}

In this scenario, we observe $D$ for each patient, as well as the actual outcome and/or censoring time $t$.  This means we know whether they had an outcome before they left the trial, as well as the time that the outcome occurred.  That time is an additional source of information, given that we can now attempt to partially order all patients that had an outcome, and ensure that people who were censored at time $t$ are estimated to live at least that long (if not longer).

To create an estimate $\widehat{\alpha}$, we use partial likelihood maximization.  That is, we construct the log likelihood function
\begin{eqnarray}\label{eq:partialLikelihood}
l(\eta) = \sum_Z \left( \eta(Z) - \log \sum_{Y\in \mathscr{R}_Z} e^{\eta(Y)}\right),
\end{eqnarray}
where $\mathscr{R}_Z = \{Y : t_Y>t_Z\}$, and $\eta(Z)$ is the argument of the exponential evaluated for patient $Z$.  In the case of the misspecified model, the argument used is $\eta(Z) = \alpha T_Z$, and the true model is $\eta^*(Z) = \alpha T_Z + Y_0(Z)$.

We borrow results from \cite{gail1984biased} about small variation of misspecified models, and \cite{tsiatis1981large} and \cite{struthers1986misspecified} about convergence rates. For notation, let
\begin{eqnarray*}
	H(y|X,T_X) = P(t>y | X,T_X).
\end{eqnarray*}

\begin{theorem}[\cite{gail1984biased}; restated]\label{thm:localHRpartial}
Let the survival model satisfy \eqref{eq:localApprox}, and the $P(T_X = 1) = p$ for $0<p<1$.  Assume we use a partial likelihood estimation of the misspecified model
\begin{eqnarray}\label{eq:misspecifiedModel}
	\lambda(t|X) &=& \lambda_0(t) e^{T_X \alpha}.
\end{eqnarray}
Let us further assume we observe both an indicator of the outcome $D$ and an outcome/censoring time $t$.  Also, let $T_0$ be the final time at which patients are observed (i.e. our censoring model censors anyone that lives past time $T_0$).

Then the partial likelihood estimate converges at a rate of $ N^{-1/2}  C_{p,\alpha,Y_0,\lambda_0,q}$ to $\alpha^*$ for some finite constant $C$ that depends on $p$, $\alpha$, $\lambda_0$, the non-treated risk model, and the censoring model, as shown by \cite{tsiatis1981large} and \cite{struthers1986misspecified}.  The estimate converges to $\alpha^*$, which satisfies
\begin{align*}
	\int_0^{T_0} &\E\left[H(y|X,T_X) T_X e^{\alpha T_X + Y_0(X)} \right]]\lambda_0(y) dy \\
	& = \int_0^{T_0} \frac{\E\left[H(y|X,T_X) T_X e^{\alpha^* T_X} \right] \E\left[H(y|X,T_X) e^{\alpha T_X + Y_0(X)} \right] }{\E\left[H(y|X,T_X) e^{\alpha^* T_X} \right]} \lambda_0(y) dy,
\end{align*}
as shown by \cite{gail1984biased}.

Moreover, if $Y_0(X)$ is well approximated locally by its first order Taylor expansion $Y_0(Z) = \mu + \beta X + \OO{(\Sigma_X)}$ for small $\beta$, then we can reduce the log term further to show
\begin{eqnarray*}
\alpha^* - \alpha \approxeq \frac{1}{2}(\beta' \Sigma_X \beta) \frac{\E(l^{(2)}) \E(T_X l^{(3)}) - \E(T_X l^{(2)}) \E( l^{(3)})}{[\E(l^{(2)})]^2 - [\E(T_X l^{(2)})]^2},
\end{eqnarray*}
where $l^{(k)}$ is the $k^{th}$ derivative of the log likelihood \eqref{eq:partialLikelihood} with respect to $\eta$.

\end{theorem}
\begin{proof}
The only part of Theorem \ref{thm:localHRpartial} that is not restated from \cite{gail1984biased} is the convergence rate.  \cite{tsiatis1981large} shows that, given a correct model for $Y_0(X)$ as in \eqref{eq:localApprox}, the partial likelihood maximization estimate $\widehat{\alpha}$ satisfies
\begin{eqnarray*}
	N^{1/2} (\widehat{\alpha} - \alpha) \rightarrow \mathscr{N}(0,\Sigma_{p,\alpha,Y_0,\lambda_0,q}),
\end{eqnarray*}
where convergence is in distribution.  Furthermore, \cite{struthers1986misspecified} shows that the same rate of convergence applies to a misspecified model \eqref{eq:misspecifiedModel}, with the exception that $\widehat{\alpha}$ converges to an a priori unknown value $\alpha^*$.  The rest of the proof focuses on characterizing $\alpha^*$ in terms of known quantities, as done by \cite{gail1984biased}.
\end{proof}

 \section{Examples}\label{applications}

 \subsection{Synthetic Randomized Drug Trial}\label{drugTrial}
 
 We create a model of synthetic patients in a drug trial.  The patient baseline model consists of 9 dimensions of correlated information, where
\begin{eqnarray*}
x_k\in[0,1], & x_{3i +1}^2 + x_{3i +2}^2 + x_{3i +3}^2 = 1, & \textnormal{ for } i\in\{0,1,2\}.
\end{eqnarray*}
We note that this model choice is arbitrary, and was solely chosen to model a dependence between patient features.  The patients are randomly split into treatment A and treatment B.

The baseline hazard function for the patients is a Weibull distribution of the type
\begin{eqnarray*}
P(X<t) =  e^{-\lambda t^k},
\end{eqnarray*}
for $\lambda=2$ and $k=1.2$.  If a patient is in treatment A, their outcome time $t_x$ is sampled from the Weibull distribution.  If a patient is in treatment B, their outcome time $t_x$ is sampled from the Weibull distribution and then adjusted by $t_x \mapsto t_x e^{\beta_x}$.  Any patient is censored if $t_x > T$ for a fixed $T$.

The key behind this model is that $\beta_x$ is patient specific, and depends only on a subset of the patient's baseline information.  Specifically,
\begin{eqnarray*}
\beta_x = f(x_1, x_2, x_3),
\end{eqnarray*}
where $f$ is displayed in Figure \ref{fig:groundTruth}.  Setting $T=2$ and $\sup |e^{\beta_x} | = 3$, about $10.5\%$ of patients have an outcome.

\begin{figure}[!h]
\begin{center}
\includegraphics[width=.4\textwidth]{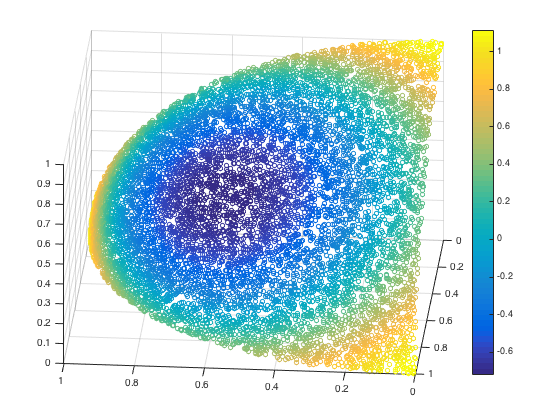}
\end{center}
\caption{}\label{fig:groundTruth}
\end{figure}

A linear Cox proportional hazard model, by definition, is unable to recover the full spread of $\beta_x$.  But beyond that, in this case a linear model fails to even recover the treatment group as a significant factor in risk.  See Table \ref{tab:linearCox} for the regression coefficients.
\begin{table}[!h]
\footnotesize
\begin{tabular}{c|rrrrrrrrrr}
{\bf Variable Name} & Treatment & $x_1$ & $x_2$ & $x_3$ & $x_4$ & $x_5$ & $x_6$ & $x_7$ & $x_8$ & $x_9$ \\
\hline
{\bf Coefficeint} & 0.0597  & -1.8517 &  -1.6500 &  -1.8440 &  -0.1797   & -0.3245  & -0.1619 &  -0.1127  &  0.0566   & 0.0954\\
{\bf p-value} & 0.3155  &  0.0001 &    0.0001  &  0.0001 &   0.3766   & 0.1087&   0.4295 &   0.5826  &  0.7854   & 0.6450
\end{tabular}
\caption{}\label{tab:linearCox}
\end{table}

Our algorithm on the other hand is able to recover both the geometry of the patients and an estimate of the personalized hazard ratio $\beta_x$.  Figure \ref{fig:octentSphereHR} shows the recovered embedding of the patients, and is colored by the estimate of $\beta_x$.
\begin{figure}[!h]
\begin{tabular}{cc}
\includegraphics[width=.4\textwidth]{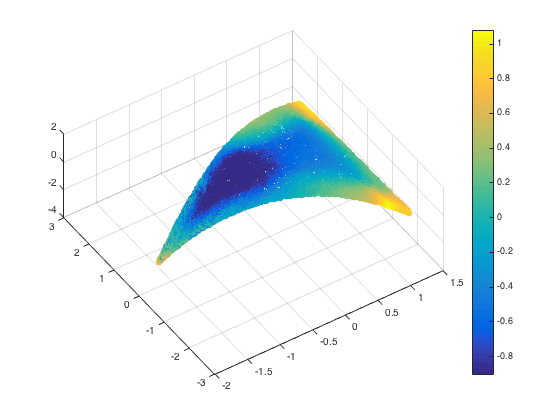} & 
\includegraphics[width=.4\textwidth]{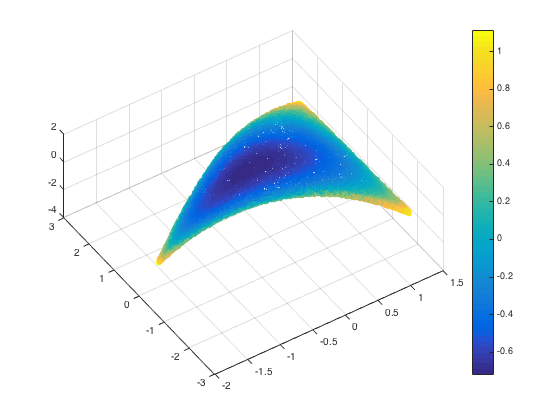} \\
(a) Predicted HR & (b) Ground Truth HR
\end{tabular}
\caption{}\label{fig:octentSphereHR}
\end{figure}
The model works for predictive personalized hazards on new testing data, as well.  We run repeated random sub-sampling validation on the toy data by retaining $80\%$ of the patients for training, and testing on the remaining $20\%$.  This was iterated $100$ times.  The results are shown in Figure \ref{fig:octentSphereValidation}.
\begin{figure}[!h]
\begin{center}
\includegraphics[width=.8\textwidth]{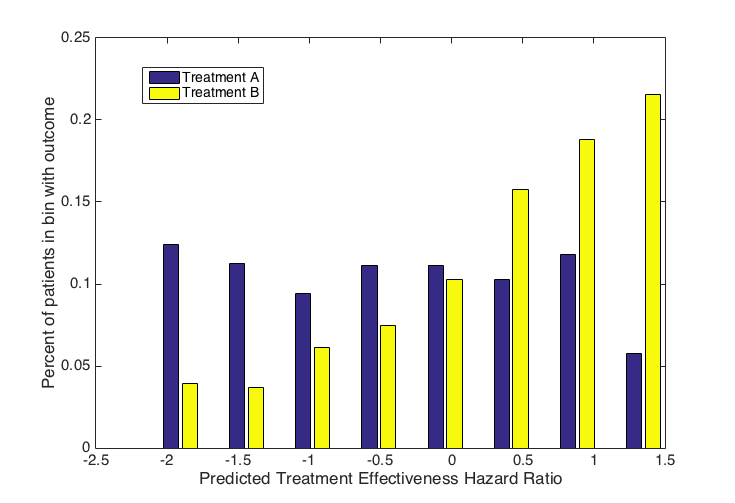} 
\end{center}
\caption{}\label{fig:octentSphereValidation}
\end{figure}



%

\subsection{Models with Treatment Propensity}
Let $X\sim \mathcal{N}(0, \Sigma)$, with $$\Sigma_{i,j} = \begin{cases} 1, & \textnormal{if } i=j \\ 0.5, & \textnormal{if } |i-j|=1\\0, & \textnormal{otherwise}\end{cases}.$$  The baseline hazard function is the same as in Section 
\ref{drugTrial}, and the personal hazard ratio is
\begin{eqnarray*}
h(X) &=& X_1 + 0.5 X_2 + 0.5 X_1 X_2 + X_\alpha X_2.
\end{eqnarray*}
However, unlike in the previous examples, $X_\alpha$ is not randomized across the population.  Instead, 
\begin{eqnarray*}
P(X_\alpha = 1) &=& P(w_x < \gamma_0 + X\gamma),
\end{eqnarray*}
where $w_x \sim \mathcal{N}(0,1)$, $\gamma_0 = 0.5$, and $\gamma = \begin{bmatrix}1& 1& 0& ...& 0\end{bmatrix}$.

\begin{figure}[!h]
\footnotesize
\begin{tabular}{ccc}
\includegraphics[width=.32\textwidth]{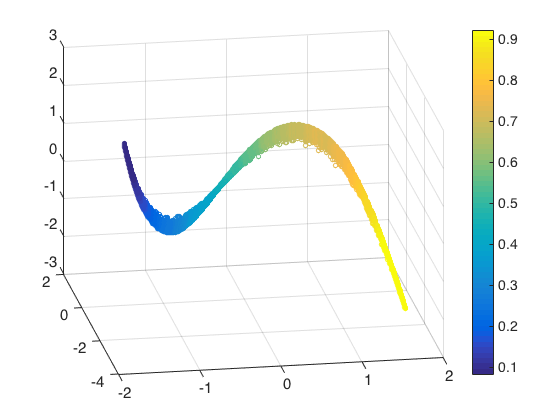} & 
\includegraphics[width=.32\textwidth]{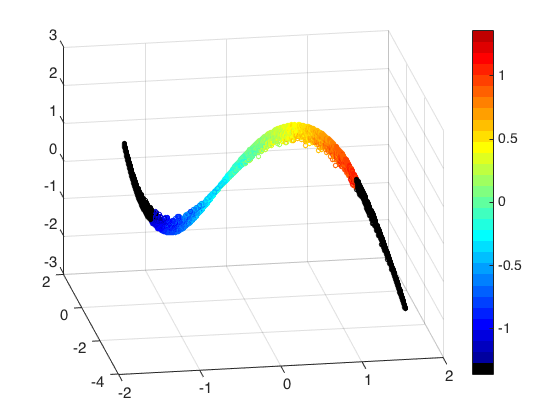} & 
\includegraphics[width=.32\textwidth]{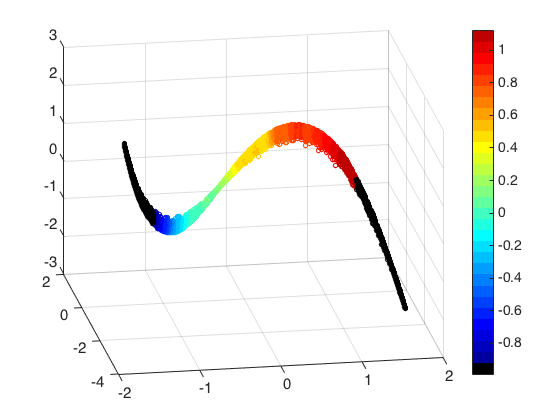} \\
Fraction Treated in Neighborhood & Ground Truth $h(X)$ & Estimated $h(X)$
\end{tabular}
\caption{Black corresponds to points where $>80\%$ of the points were from the same treatment group, and thus removed due to lack of estimate precision.}\label{fig:propOutcome}
\end{figure}
 
 We also consider a random model, in which $\Sigma$ is a random symmetric positive definite matrix with condition number less than 10, and the personal hazard ratio follows the form
 \begin{eqnarray*}
 h(X) = \sum_i \xi_i X_i  + \sum_{i,j} \eta_{i,j} X_i X_j  +  X_\alpha \left(\sum_i \nu_i X_i  + \sum_{i,j} \delta_{i,j}  X_i X_j\right),
 \end{eqnarray*}
 where $\xi_i$ and $\nu_i$ are sparse standard normal random variables which are non-zero with probability $0.5$.  Also, $\eta_{i,j}$ (resp. $\delta_{i,j}$) are standard normal random variables which are non-zero if and only if $\xi_i$ and $\xi_j$ (resp. $\nu_i$ and $\nu_j$) are non-zero.  Also, the probability that a patient is treated is determined by 
 \begin{eqnarray*}
 P(X_\alpha = 1) &=& P(w_x < \gamma_0 + X\gamma).
 \end{eqnarray*}
 
 We run this model across 100 iterations, where we generate $2,000$ patients who's baseline hazard function is drawn from a Weibull distribution as in previous examples.  The patients are then censored such that $\epsilon$ fraction of the patients have an outcome, where $\epsilon$ is a uniform random variable drawn from $[1/3, 1]$.  We calculate the correlation between the predicted personalized treatment effect $\widehat{f}(x)$ and the ground truth treatment effect $\sum_i \nu_i X_i  + \sum_{i,j} \delta_{i,j}  X_i X_j$.  Because of the propensity for treatment, we only estimate $\widehat{f}(x)$ in neighborhoods such that $\le 80\%$ of the patients are in the same treatment group.
  The histogram of correlations is in Figure \ref{fig:prop100Iters}.
 
 \begin{figure}[!h]
\begin{center}
\includegraphics[width=.6\textwidth]{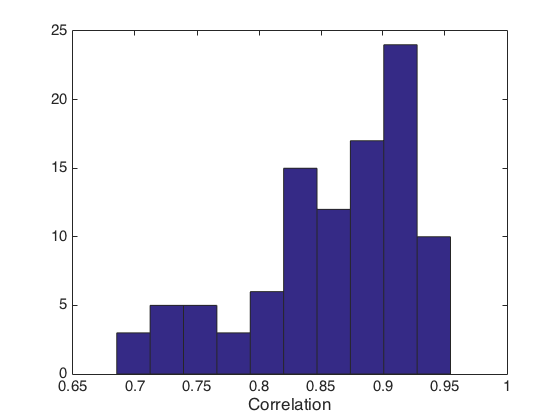} 
\end{center}
\caption{Histogram of Correlations between predicted personalized treatment effect and ground truth across 100 iterations.}\label{fig:prop100Iters}
\end{figure}

 \subsection{Real World Data}
 We examine our algorithm on breast cancer data from the Rotterdam Tumor Bank \cite{}.  The Rotterdam tumor bank dataset contains records for 1,546 patients with node-positive breast cancer, and nearly 90 percent of the patients have an observed outcome.  Because this data has no ground truth, we must use leave out cross-validation to validate the recommendations.  We train the model on a random $80\%$ of the patients, and test on $20\%$, and we iterate this process 100 times.  We then split the testing data into three groups, where $\sigma(f)$ is the standard deviation of treatment recommendations:
 \begin{itemize}
 	\item {\bf Recommended:} People with recommendation $|f(x)| > c\cdot \sigma(f)$ such that $f(x)>0$ and $X_\alpha=0$ or  $f(x)<0$ and $X_\alpha=1$.  These are people whose actions followed the recommendation.
 	\item {\bf Neutral:} People with recommendation $|f(x)|<c\cdot \sigma(f)$.  These are people without a strong recommendation.
 	\item {\bf Anti-Recommended:} People with recommendation $|f(x)| > c\cdot \sigma(f)$ such that $f(x)<0$ and $X_\alpha=0$ or  $f(x)>0$ and $X_\alpha=1$.  These are people whose actions went against the recommendation.
 \end{itemize}

We then plot the survival curves of the Recommended and Anti-Recommended groups in Figure \ref{fig:rotterdamCurves}.  Again, these were all testing samples in order to avoid over-fitting to the outcomes.  The group of testing data patients that followed the recommendations lived significantly longer than those that did not follow the recommendation, with a p-value of $p=0.00025$ for the log-rank test of whether these curves are significantly different.

\begin{figure}
	\footnotesize
	\begin{tabular}{c}
		\includegraphics[width=.7\textwidth]{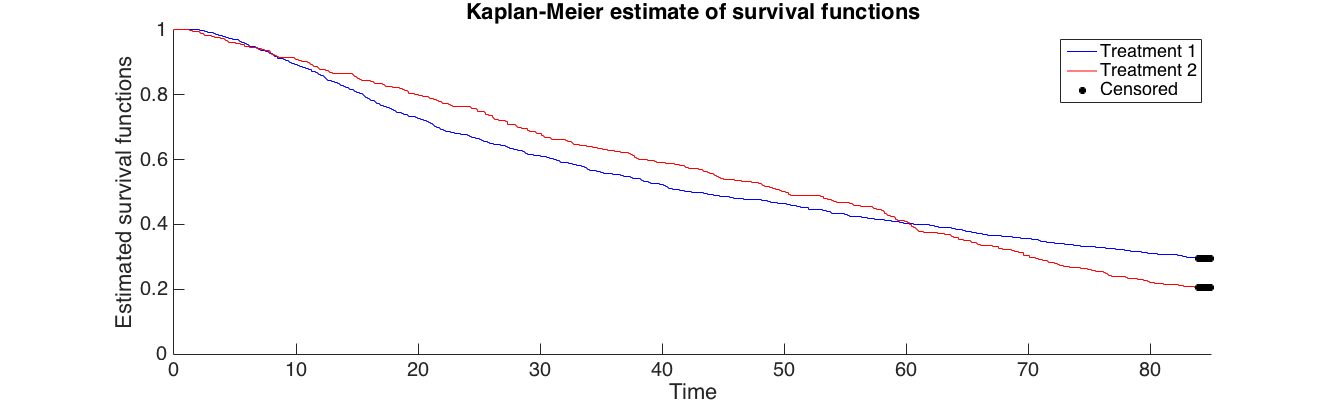} \\
		Treatment survival curves \\
		\includegraphics[width=.7\textwidth]{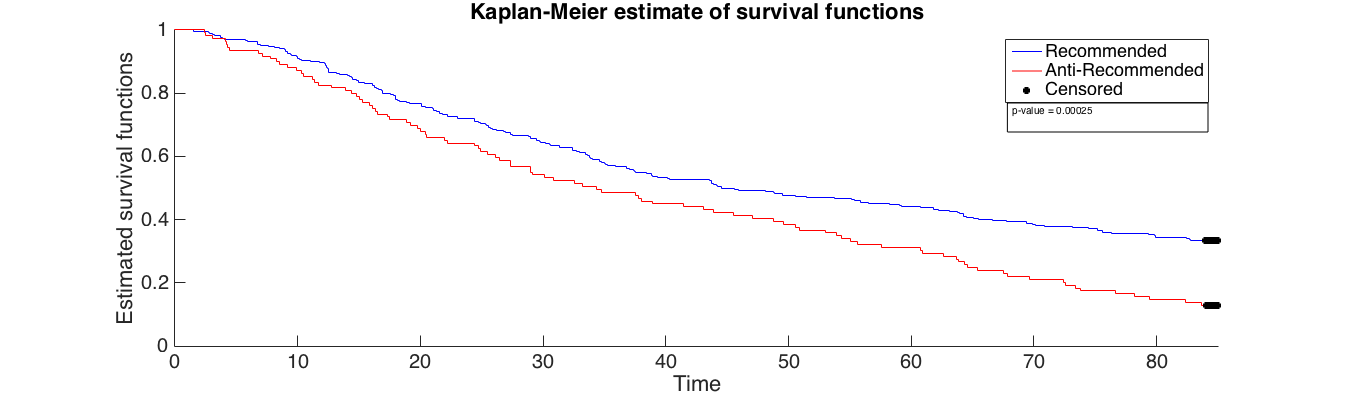} \\
		Recommendation survival curves
	\end{tabular}
\caption{Treatment recommendation on Rotterdam breast cancer testing data for $c=0.5$.}\label{fig:rotterdamCurves}
\end{figure}

 \section{Conclusions and Future Work}
This paper develops a method for building a data dependent metric $\rho: X\times X \rightarrow \R^+$ that is simultaneously learns the level sets of a function $F$.  The method only needs to evaluate $F$ on various half spaces of the data, making it useful when $F$ cannot be evaluated pointwise.  We develop a weighted tree distance to accomplish this, and  develop feature weights at multiple scales and locations in the data.  Once $\rho$ has been discovered, we can do k nearest neighbor prediction for new points added to the data without knowledge of $F$ at the point.

This algorithm was designed with medical applications in mind, specifically building a local cox proportional hazard model for patients in a dataset.  The embedding created by $\rho$ can be used to characterize types of people that are hurt or helped by a drug, and even assign a personalized treatment hazard score to new patients whose outcomes are unknown.


\section*{Acknowledgements}
The author would like to thank Raphy Coifman and Jonathan Bates for many discussions about the problem, and Harlan Krumholz, and Shu-Xia Li for introducing the issues associated with treatment effectiveness and Cox proportional hazard models.  The author is supported by NSF Award No. DMS-1402254.

\bibliographystyle{plain.bst}	
\bibliography{personalizedHRQuestionnaire.bib}		

\appendix

\section{Proof of Theorem \ref{thm:posdefKernel}} \label{appendix:posdefKernel}

	We show $k$ can be expressed as an integral over all ambient space via
	\begin{eqnarray*}
		k(x,y) &=& \int dz \frac{exp\{-(x-z)^\intercal W_x^{-1}(x-z)/\sigma^2\}}{\sqrt{det(W_x)}} \frac{exp\{-(y-z)^\intercal W_y^{-1}(y-z)/\sigma^2\}}{\sqrt{det(W_y)}}.
	\end{eqnarray*} 
	We then use identities 8.1.7 and 8.1.8 from \cite{matrixCookbook}, which show the product of two gaussians gives 
	\begin{eqnarray*}
		\frac{C}{\sqrt{det(W_x + W_y)}} \textnormal{exp}\left[ -(x - y)^{\intercal} (W_x + W_y)^{-1} (x-y) / \sigma^2 \right] \cdot \frac{e^{-m^\intercal W^{-1} m}}{\sqrt{det(W)}}.
	\end{eqnarray*}
	where $m$ and $W$ are combinations of $x$, $y$, and $z$.  Their exact forms are irrelevant, as the right hand term is simply a normalized guassian that can be integrated out with respect to $z$.  Thus, after evaluating the integral, we are left with 
	\begin{eqnarray*}
		k(x, y) = C\cdot \frac{e^{- (x - y)^\intercal ( W_{x} + W_{y} )^{-1} (x - y) / \sigma^2}}{\sqrt{\det(W_{x}+ W_{y})}}.
	\end{eqnarray*}
	
	Now, for any $w(x)$ we can compute
	\footnotesize
	\begin{eqnarray*}
		\int dx dy\textnormal{ }  w(x) w(y) k(x,y) &=& \int dx dy\textnormal{ } w(x) w(y) \int dz \frac{exp\{-(x-z)^\intercal W_x^{-1}(x-z)/\sigma^2\}}{\sqrt{det(W_x)}} \frac{exp\{-(y-z)^\intercal W_y^{-1}(y-z)/\sigma^2\}}{\sqrt{det(W_y)}} \\
		&=& \int dz \left(\int dx \textnormal{ } w(x)  \frac{exp\{-(x-z)^\intercal W_x^{-1}(x-z)/\sigma^2\}}{\sqrt{det(W_x)}} \right) \left( \int dy \textnormal{ } w(y) \frac{exp\{-(y-z)^\intercal W_y^{-1}(y-z)/\sigma^2\}}{\sqrt{det(W_y)}}\right) \\
		&=& \int dz \textnormal{ }  \left(\int dx \textnormal{ } w(x)  \frac{exp\{-(x-z)^\intercal W_x^{-1}(x-z)/\sigma^2\}}{\sqrt{det(W_x)}} \right)^2\\
		&\ge& 0.
	\end{eqnarray*}
	\normalsize

\section{Proof of Theorem \ref{thm:localHR}}\label{appendix:localHRProof}

Let the survival model satisfy 
\begin{eqnarray*} 
	\lambda(t|X) &=& \lambda_0(t) e^{T_X \alpha + Y_0(X)},
\end{eqnarray*}
and the $P(T_X = 1) = p$ for $0<p<1$.  Assume we use likelihood estimation of the misspecified model
\begin{eqnarray*} 
	\lambda(t|X) &=& \lambda_0(t) e^{T_X \alpha}.
\end{eqnarray*}
Also, assume the trial is observed until time $T_0$.

The rate of convergence of for the method of moments calculation \eqref{eq:methodMoments} is a simple application of the central limit theorem, as $D_X$ is a Bernoulli random variable whose probability is a function of the number of samples $P(T_X=1) N$, and the rate of outcomes prior to censoring, which is dictated by $\alpha$, $\lambda_0$, $Y_0$ and the censoring model $q$.  The rest of the proof focuses on characterizing $\alpha^*$ in terms of known quantities.

\cite{gail1984biased} show that the method of moments limit point $\alpha^*$ satisfy
\begin{align*}
\E \left[D - e^{\alpha^* T} \int_0^{T_0}\lambda_0(t) dt \right] = 0 \\
\E \left[T_X \left(D - e^{\alpha^* T}\int_0^{T_0}\lambda_0(t) dt \right)\right] = 0
\end{align*}
under the false model.  Rearranging these equations and noting that $$\Pi(X,T_X) = e^{\alpha T_X + Y_0(X)} \E\left[\int_0^{T_0}\lambda_0(t) dt \middle| X, T_X  \right],$$ we arrive at the equations
	\footnotesize 
\begin{eqnarray*}
	p \Pi(1)  + (1-p) \Pi(0) &=& p e^{\frac{\alpha^*}{2} T_X} \E_X \left[ \Pi(X,1) e^{-\frac{\alpha}{2} T_X - Y_0(X)}  \right] + (1-p) e^{-\frac{\alpha^*}{2} T_X} \E_X \left[ \Pi(X,0) e^{\frac{\alpha}{2} T_X - Y_0(X)}  \right] \\
	p \Pi(1)  - (1-p) \Pi(0) &=& p e^{\frac{\alpha^*}{2} T_X} \E_X \left[ \Pi(X,1) e^{-\frac{\alpha}{2} T_X - Y_0(X)}  \right] - (1-p) e^{-\frac{\alpha^*}{2} T_X} \E_X \left[ \Pi(X,0) e^{\frac{\alpha}{2} T_X - Y_0(X)}  \right].
\end{eqnarray*}
\normalsize
Solving these equations for $\alpha^*$ yields the desired result.

 \end{document}